\documentclass{article}
\usepackage[utf8]{inputenc}
\usepackage{style}
\usepackage{booktabs}
\usepackage{xfrac}
\usepackage{xcolor}
\usepackage{float}
\RequirePackage{natbib}

\renewcommand{\cite}[1]{\citep{#1}}
\usepackage{hyperref}
\definecolor{mydarkblue}{rgb}{0,0.08,0.45}
\hypersetup{ %
pdftitle={},
pdfauthor={},
pdfsubject={},
pdfkeywords={},
pdfborder=0 0 0,
pdfpagemode=UseNone,
colorlinks=true,
linkcolor=mydarkblue,
citecolor=mydarkblue,
filecolor=mydarkblue,
urlcolor=mydarkblue,
pdfview=FitH}

\title{\vspace{-1cm}Learning Rate Annealing Can Provably Help Generalization,\\Even for Convex Problems}
\author{Preetum Nakkiran\\
Harvard University\footnote{Part of this work completed while the author
was interning at OpenAI.}
}
\date{}

\newcommand{\cD}{\mathcal{D}}

\newcommand{\bs}{\beta^*}

\def\shownotes{1}  \ifnum\shownotes=1
\newcommand{\authnote}[2]{[ #2 --#1 ]}
\else
\newcommand{\authnote}[2]{}
\fi

\begin{document}

\maketitle
\vspace{-0.5cm}
\begin{abstract}
Learning rate schedule can significantly affect
generalization performance in modern neural networks,
but the reasons for this are not yet understood.
\citet{li2019explaining} recently proved this behavior
can exist in a simplified non-convex neural-network setting.
In this note, we show that this phenomenon can exist even for convex learning problems -- in particular, linear regression in 2 dimensions.

We give a toy convex problem where learning rate annealing (large initial learning rate, followed by small learning rate)
can lead gradient descent to minima with provably better generalization than using a small learning rate throughout.
In our case, this occurs due to a combination of
the mismatch between the test and train loss landscapes,
and early-stopping.
\end{abstract}

\section{Introduction}
The learning rate schedule of stochastic gradient descent
is known to strongly affect the generalization of modern neural networks,
in ways not explained by optimization considerations alone.
In particular, training with a large initial learning-rate followed by a smaller annealed learning rate
can vastly outperform training with the smaller learning rate throughout -- even
when allowing both to reach the same train loss.
The recent work of~\citet{li2019explaining} sheds some light on this, by showing a
simplified neural-network setting in which this behavior provably occurs.

\begin{figure}[H]
    \centering
    \begin{subfigure}[b]{0.3\textwidth}
        \includegraphics[width=\textwidth]{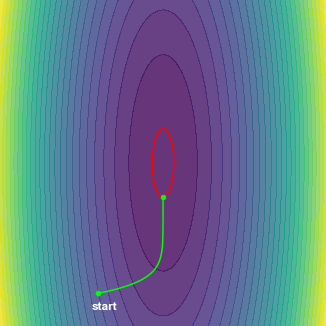}
        \caption{Small learning rate.}
        \label{fig:flow}
    \end{subfigure}
    ~
    \begin{subfigure}[b]{0.3\textwidth}
        \includegraphics[width=\textwidth]{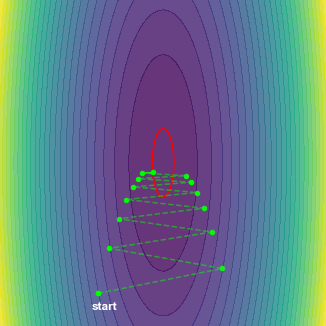}
        \caption{Large, then small learning rate.}
        \label{fig:gd}
    \end{subfigure}
    ~
    \begin{subfigure}[b]{0.3\textwidth}
        \includegraphics[width=\textwidth]{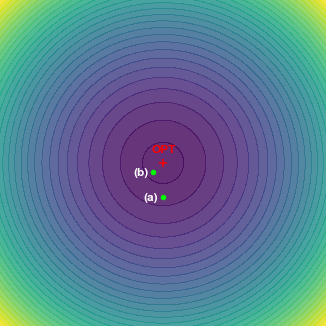}
        \caption{Test loss landscape.}
        \label{fig:dist}
    \end{subfigure}
    \caption{{\bf How learning rate can affect generalization.} Minima with the same train loss can have different test losses, if the empirical loss landscape is distorted relative to the population loss (e.g. due to undersampling/overparameterization).
    Figures \ref{fig:flow} and \ref{fig:gd} show gradient descent on the train loss, with different learning rate schedules, stopping at the same small value of train-loss (red ellipse).
    Note that the minima found in \ref{fig:flow} is worse than \ref{fig:gd} in test loss (shown in Figure \ref{fig:dist}), through they have identical train loss.
    }\label{fig:main}
\end{figure}

It may be conjectured that non-convexity is crucial for learning-rate schedule to affect generalization (for example, because
strongly-convex objectives have unique global minima).
However, we show this behavior can appear even in strongly-convex learning problems.
The key insight is that although strongly-convex problems have a unique minima,
early-stopping breaks this uniqueness: there is a set of minima with the same train loss $\eps$.
And these minima can have different generalization properties
when the train loss surface does not closely approximate the test loss surface.
In our setting, a large learning-rate prevents
gradient descent from optimizing along high-curvature directions,
and thus effectively regularizes against directions that are
high-curvature in the train set, but low-curvature on the true distribution
(see Figure~\ref{fig:main}).

Technically, we model a small learning rate by considering gradient flow.
Then we compare the following optimizers:
\begin{enumerate}[(A)]
    \item Gradient descent with an infinitesimally small step size (i.e. gradient flow).
    \item Gradient descent with a large step size, followed by gradient flow.
\end{enumerate}
We show that for a particular linear regression problem,
if we run both (A) and (B) until reaching the same $\eps > 0$ train loss,
then with probability $3/4$ over the train samples,
the test loss of (A) is twice that of (B).
That is, starting with a large learning rate and then annealing to
an infinitesimally small one helps generalization significantly.

{\bf Our Contribution.}
We show that non-convexity is not required to reproduce
an effect of learning rate on generalization that is observed
in deep learning models.
We give a simple model where the mechanisms behind this can be theoretically understood, which shares some features with deep learning models in practice.
We hope such simple examples
can yield insight and intuition that may eventually
lead to a better understanding of the effect of learning rate
in deep learning.

{\bf Organization.}
In Section~\ref{sec:main}
we describe the example.
In Section~\ref{sec:discuss}
we discuss features of the example, and its potential implications
(and non-implications) in deep learning.
In Appendix~\ref{sec:formal}, we include formal proofs and provide
a mild generalization of the example in Section~\ref{sec:main}.

\subsection{Related Work}
This work was inspired by \citet{li2019explaining},
which gives a certain simplified neural-network setting
in which learning-rate schedule
provably affects generalization, despite performing identically with respect to optimization.
The mechanisms and intuitions in \citet{li2019explaining} depend crucially on
non-convexity of the objective.
In contrast, in this work we provide complementary results, by showing that this behavior can occur even in convex models
due to the interaction between early-stopping and
estimation error.
Our work is not incompatible with \citet{li2019explaining}; indeed,
the true mechanisms behind generalization in real neural networks may involve both of these factors (among others).

There are many empirical and theoretical works
attempting to understand the effect of learning-rate in deep learning,
which we briefly describe here.
Several recent works study the effect of learning rate
by considering properties (e.g. curvature) of the loss landscape,
finding that different learning rate schedules can lead SGD
to regions in parameter space with different
local geometry~\cite{jastrzebski2020break, lewkowycz2020large}.
This fits more broadly into a study of the implicit bias of SGD.
For example, there is debate on whether generalization is connected
to ``sharpness'' of the minima (e.g. \citet{hochreiter1997long, keskar2016large, dinh2017sharp}).
The effect of learning-rate is also often also studied alongside the effect of batch-size, since the optimal choice of these
two parameters is coupled in practice \cite{krizhevsky2014one, goyal2017accurate, smith2017don}.
Several works suggest that  the 
\emph{stochasticity} of SGD is crucial to the effect of learning-rate,
in that different learning rates lead to different stochastic dynamics,
with different generalization behavior
\cite{smith2017bayesian,mandt2017stochastic, hoffer2017train, mccandlish2018empirical, welling2011bayesian,li2019explaining}. In particular, it may be the case that the stochasticity in the gradient estimates of SGD acts as an effective ``regularizer'' at high learning rates.
This aspect is not explicitly present in our work, and is an interesting
area for future theoretical and empirical study.


\section{Convex Example}
\label{sec:main}

The main idea is illustrated in the following linear regression problem,
as visualized in Figure~\ref{fig:main}.
Consider the data distribution $\cD$ over $(x, y) \in \R^2 \x \R$ defined as:
$$
x \in \{\vec{e_1}, \vec{e_2}\} \text{ uniformly at random };
\quad
y = \langle \beta^*, x \rangle
$$
for some ground-truth $\beta^* \in \R^2$.
We want to learn a linear model
$\hat{y}(x) := \langle \beta, x \rangle$ with small mean-squared-error on the population distribution.
Specifically, for model parameters
$\beta  = (\beta_1, \beta_2) \in \R^2$, the population loss is
$$
L_\cD(\beta) := \E_{x,y \sim D}[(\langle \beta, x \rangle - y)^2]
$$
We want to approximate $\beta^* = \argmin_\beta L_\cD(\beta)$.
Suppose we try to find this minima by drawing $n$ samples $\{x_i, y_i\}$
from $\cD$,
and performing gradient descent
on the empirical loss:
$$
\hat L_n(\beta) :=
\frac{1}{n}\sum_{i \in [n]} (\langle \beta, x_i \rangle - y_i)^2
$$
starting at $\beta=0$, and stopping when $\hat L_n \leq \eps$ for some small $\eps$.

Now for simplicity let $n = 3$,
and let the ground-truth model be $\beta^* = (\beta^*_1, \beta^*_2) \in \R^2$.
With probability $3/4$, two of the samples will have the same value of $x$ -- say this value is $x = \vec{e_1}$,
so the samples $(x_i, y_i)$ are $\{(e_1, \beta^*_1), (e_1, \beta^*_1), (e_2, \beta^*_2)\}$.
In this case the empirical loss is
\begin{equation}
\label{eqn:trainloss}
\hat L_n(\beta)
=
\underbrace{\frac{2}{3}(\beta_1 - \beta^*_1)^2}_{\textrm{(A)}}
+
\underbrace{\frac{1}{3}(\beta_2 - \beta^*_2)^2}_{\textrm{(B)}}
\end{equation}

which is \emph{distorted} compared to the population loss,
since we have few samples relative to the dimension.
The population loss is:
\begin{equation}
\label{eqn:testloss}
L_\cD(\beta)
=
\frac{1}{2}(\beta_1 - \beta^*_1)^2
+
\frac{1}{2}(\beta_2 - \beta^*_2)^2
\end{equation}

The key point is that although the global minima of $L_\cD$ and $\hat L_n$ are identical,
their level sets are not: not all points of small train loss
$\hat L_n(\beta)$ have identical
test loss $L_\cD(\beta)$.
To see this, refer to Equation~\ref{eqn:trainloss},
and consider two different
ways that train loss $\hat{L}_n = \eps$ could be achieved.
In the ``good'' situation,
term (A) in Equation~\ref{eqn:trainloss} is $\eps$,
and term (B) is $0$.
In the ``bad'' situation,
term (A) is $0$ and term (B) is $\eps$.
These two cases have test losses which differ by a factor of two,
since terms are re-weighted differently in the test loss
(Equation~\ref{eqn:testloss}). This is summarized in the below table.
\begin{center}
\begin{tabular}{l l c c }\toprule
& Residual $(\beta - \beta^*)$ & Train Loss $\hat{L}_n(\beta)$ & Test Loss $L_\cD(\beta)$ \\
\midrule
Good: & $[\sqrt{\frac{3\eps}{2}}, 0]$ & $\eps$ & $0.75\eps$ \\
Bad: & $[0, \sqrt{3\eps}]$ & $\eps$ & $1.5\eps$ \\
\end{tabular}
\end{center}
Now, if our optimizer stops at train loss $\eps$,
it will pick one of the points in $\{\beta : \hat{L}_n(\beta)= \eps\}$.
We see from the above that some of these points are twice as bad as others.

Notice that \emph{gradient flow} on $\hat L_n$ (Equation~\ref{eqn:trainloss})
will optimize twice as fast along the $\beta_1$ coordinate compared to $\beta_2$. The gradient flow
dynamics of the residual $(\beta - \beta^*)$ are:
\[
\frac{d\beta}{dt} = -\nabla_\beta \hat{L}_n
\implies
\begin{cases}
\frac{d}{dt}(\beta_1 - \beta_1^*) = -\frac{4}{3}(\beta_1 - \beta_1^*)\\
\frac{d}{dt}(\beta_2 - \beta_2^*) = -\frac{2}{3}(\beta_2 - \beta_2^*)
\end{cases}
\]
This will tend to find solutions closer to the ``Bad'' solution from the above table, where
$(\beta_1 - \beta_1^*)^2 \approx 0$
and
$(\beta_2 - \beta_2^*)^2 \gg 0$.

However, \emph{gradient descent} with a large step size
can oscillate on the $\beta_1$ coordinate, and
achieve low train loss by optimizing $\beta_2$ instead.
Then in the second stage, once the learning-rate is annealed,
gradient descent will optimize on $\beta_1$ while keeping
the $\beta_2$ coordinate small.
This will lead to a minima closer to the ``Good'' solution,
where $(\beta_2 - \beta_2^*) \approx 0$.
These dynamics are visualized in Figure~\ref{fig:main}.

This example is formalized in the following claim.
The proof is straightforward, and included
in Appendix~\ref{sec:formal}.
\begin{claim}
\label{corr}
For all $0 < \alpha < 1$, 
there exists a distribution $\cD$ over $(x, y) \in \R^2 \x \R$
and a learning-rate $\eta > 0$ (the ``large'' learning-rate)
such that for all $ 0 < \eps < 1$, the following holds.
With probability $3/4$ over $n=3$ samples:
\begin{enumerate}
    \item Gradient flow from $0$-initialization,
    early-stopped at train loss $\eps$,
    achieves population loss
    $$L_\cD(\beta_{\emph{bad}}) \geq \frac{3}{2}\eps(1-\alpha)$$ 
    \item Annealed gradient descent from $0$-initialization (i.e. gradient descent with stepsize $\eta$ for $K$ steps, followed by gradient flow)
    stopped at train loss $\eps$,
    achieves population loss
    $$L_\cD(\beta_{\emph{good}}) \leq \frac{3}{4} \eps + \exp(-\Omega(K))$$
\end{enumerate}

In particular, since $\alpha$ can be taken arbitrarily small, and $K$ taken arbitrarily large,
this implies that gradient flow achieves a population loss
twice as high as
gradient descent with a careful step size, followed by gradient flow.

Moreover, with the remaining $1/4$ probability,
the samples will be such that gradient flow and annealed gradient descent
behave identically:
$$
\beta_{\emph{good}}
= 
\beta_{\emph{bad}}
$$
\end{claim}

\section{Discussion}
\label{sec:discuss}

{\bf Similarities.} This example, though stylized, shares several features
with deep neural networks in practice.
\begin{itemize}
    \item  Neural nets trained with cross-entropy loss cannot
be trained to global emperical risk minimas,
and are instead early-stopped at some small value of train loss.
\item There is a mismatch between train and test loss landscapes (in deep learning, this is due to overparameterization/undersampling).
\item Learning rate annealing typically generalizes better than using a small constant learning rate~\cite{goodfellow2016deep}.
\item The ``large'' learning rates used in practice 
are far larger than what optimization theory would prescribe.
In particular, consecutive iterates of SGD
are often negatively-correlated with each other in the later stage of training \cite{xing2018walk},
suggesting that the iterates are oscillating around a sharp valley.
\citet{jastrzebski2018relation} also finds that the typical SGD step
is too large to optimize along the steepest directions of the loss.
\end{itemize}

{\bf Limitations.}
Our example is nevertheless a toy example,
and does not exhibit some features of real networks.
For example, our setting has only one basin of attraction.
But real networks have many basins of attraction,
and there is evidence that a high initial learning rate
influences the choice of basin 
(e.g. \citet{li2019explaining,jastrzebski2020break}).
It remains an important open question to understand
the effect of learning rate in deep learning.

\subsubsection*{Acknowledgements}
We thank John Schulman for a discussion around learning rates
that led to wondering if this can occur in convex problems.
We thank Aditya Ramesh, Ilya Sutskever, and Gal Kaplun for helpful comments on presentation,
and Jacob Steinhardt for technical discussions refining these results.

Work supported in part by the Simons Investigator Awards of Boaz Barak
and Madhu Sudan, and NSF Awards under grants
CCF 1565264, CCF 1715187, and CNS 1618026.

\bibliographystyle{apalike}
\bibliography{refs}

\appendix
\section{Formal Statements}
\label{sec:formal}

In this section, we state and prove 
Claim~\ref{corr}, along with a mild generalization
of the setting via Lemma~\ref{thm:main}.

\subsection{Notation and Preliminaries}
For a distribution $\cD$ over $(x, y) \in \R^d \x \R$,
let $L_\cD(\beta) := \E_{x,y \sim D}[(\langle x, \beta \rangle - y)^2]$
be the population loss for parameters $\beta$.
Let $\hat L_n(\vec{\beta})
= \frac{1}{n}\sum_{i} (\langle x_i, \beta \rangle - y_i)^2$
be the train loss for $n$ samples $\{(x_i, y_i)\}$.

Let $GF_{\beta_0}(L, \eps)$ 
be the function
which optimizes train loss $L$
from initial point $\beta_0$,
until the train loss reaches
early-stopping threshold $\eps$,
and then outputs the resulting parameters.

Let $GD_{\beta_0}(L, K, \tau, \eps)$ be the function which first optimizes
train loss $L$ with gradient-descent 
starting at initial point $\beta_0$, with step-size $\eta$,
for $K$ steps, and
an early-stopping threshold $\eps$,
and then continues with gradient flow with early-stopping threshold $\eps$.

For $n$ samples, let $X \in \R^{n \x d}$ be the design matrix of samples.
Let $\Sigma_X = \E_{x, y \sim \cD}[xx^T]$ be the population covariance
and let
$\hat{\Sigma}_X = \frac{1}{n}X^TX$ be the emperical covariance.

We assume throughout
that $\Sigma_X$ and $\hat{\Sigma}_X$ are simultaneously diagonalizable:
$$
\Sigma_X = U \Lambda U^T
\quad ; \quad
\hat\Sigma_X = U \Gamma U^T
$$
for diagonal $\Lambda = \textrm{diag}(\lambda_i), S = \textrm{diag}(\gamma_i)$ and orthonormal $U$.
Further, assume without loss of generality that $\hat\Sigma_X \succ 0$
(if $\hat\Sigma_X$ is not full rank, we can restrict attention to its image).

We consider optimizing the the train loss:
$$\hat{L}(\beta) = \frac{1}{n}||X(\beta - \bs)||_2^2 = ||\beta - \bs||^2_{\hat \Sigma_X}$$

Observe that the optimal population loss for a point with $\eps$ train loss is:
\begin{align}
\min_{\substack{
\text{s.t. } L_{\hat{\Sigma}}(\beta) = \eps
}}
L_{\Sigma}(\beta)
= \eps \min_{i}(\{\frac{\lambda_i}{\gamma_i}\})
\end{align}
Similarly, the worst population loss for a point with $\eps$ train loss is:
\begin{align}
\max_{\substack{
\text{s.t. } L_{\hat{\Sigma}}(\beta) = \eps
}}
L_{\Sigma}(\beta)
= \eps \max_{i}(\{\frac{\lambda_i}{\gamma_i}\})
\end{align}

Now, the gradient flow dynamics on the train loss is:
\begin{align*}
    \frac{d\beta}{dt} = - \nabla_\beta \hat L(\beta)
    = -\frac{2}{n} X^TX(\beta - \beta^*)
\end{align*}

Switching to coordinates $\delta := U^T(\beta - \beta^*)$, these dynamics are
equivalent to:
\begin{align*}
    \frac{d\delta}{dt} = -2\Gamma \delta
\end{align*}
The empirical and population losses in these coordinates can be written as:
\begin{align*}
    L_\cD(\delta) = \sum_i \lambda_i \delta_i^2 \quad ; \quad
    \hat L(\delta) = \sum_i \gamma_i \delta_i^2
\end{align*}
And the trajectory of gradient flow, from initialization $\delta(0) \in \R^d$, is
\begin{align*}
    \forall i \in [d]: ~~\delta_i(t) = \delta_i(0)e^{-2\gamma_i t}
\end{align*}

With this setup, we now state and prove the following main lemma,
characterizing the behavior of gradient flow and gradient descent
for a given sample covariance.

\subsection{Main Lemma}
\begin{lemma}
\label{thm:main}
Let $\Sigma, \hat\Sigma \in \R^{d \x d}$ be as above.
Let eigenvalues be ordered 
according to $\hat\Sigma$: $\gamma_1 \geq \gamma_2 \geq \dots \geq \gamma_d$.

Let $S = \{k, k+1, \dots, d\}$ for some $k$. That is, let $S$ index a block of ``small'' eigenvalues.
Let $\bar{S} = [n] \setminus S$.
Assume there is an \emph{eigenvalue gap} between the ``large''
and ``small'' eigenvalues, where for some $p > 0$
$$\forall i \in \bar S: \gamma_i / \gamma_k \geq 1+p.$$

For all $0 < \alpha < 1$, if the initialization
$\delta(0) = U^T(\beta_0 - \bs) \in \R^d$
satisfies

\begin{enumerate}
    \item At initialization,
    the contribution of the largest eigenspace to the train loss is at least $\eps$:
$$\gamma_1 \delta_1(0)^2 > \eps$$

\item
The eigenvalue gap is large enough to ensure that eventually
only the ``small'' eigenspace is significant, specifically:

$$
\eps^p \sum_{i \in \bar S}
\frac{\gamma_i \delta_i(0)^2}{(|S| \gamma_{j^*} \delta_{j^*}(0)^2)^{1+p}}
\leq \alpha
\quad\text{where }
j^* := \argmin_{j \in S} \gamma_j \delta_j(0)^2
$$
\end{enumerate}

Then there exists a learning-rate $\eta := 1/\gamma_1$
(the ``large'' learning-rate)
such that
\begin{enumerate}
    \item Gradient flow $\beta_{slow} \gets GF_{\vec \beta_0}(\hat L, \eps)$
    achieves population loss
    $$L_\cD(\beta_{slow}) \geq
    \eps(1-\alpha)
    \min_{j \in S}
    \frac{\lambda_j}{\gamma_j}
    $$ 
    \item Annealed gradient descent run for $K$ steps
    $\beta_{fast} \gets GD(\hat L, K, \eta, \eps)$
    has loss
    $$L_\cD(\beta_{fast}) \leq
    \eps \max_{j: \gamma_j = \gamma_1} (\frac{\lambda_j}{\gamma_j}) + \exp(-\Omega(K))$$
    In particular, if the top eigenvalue of $\Gamma$ is unique, then
    $$L_\cD(\beta_{fast}) \leq
    \eps \frac{\lambda_1}{\gamma_1} + \exp(-\Omega(K))$$
    The constant $\Omega(K)$ depends on the spectrum, and
    taking $\Omega(K) = -K\log(\max_{j : \gamma_j \neq \gamma_1} | 1 - 2 \gamma_j / \gamma_1 |)$ suffices.
\end{enumerate}

\end{lemma}

\begin{proof}

{\bf Gradient Flow.}

The proof idea is to show that gradient flow must
run for some time $T$ in order to reach $\eps$ train loss -- and
since higher-eigenvalues optimize faster, most of
the train loss will be due to contributions from the ``small''
coordinates $S$.

Let $T$ be the stopping-time of gradient flow, such that $\hat L(\delta(T)) = \eps$.
We can lower-bound the time $T$ required to reach $\eps$ train loss as:
\begin{align}
|S|\gamma_{j^*}\delta_{j^*}(0)^2 e^{-4\gamma_k T}
&\leq
\sum_{j \in S}
\gamma_{j}\delta_{j}(0)^2 e^{-4\gamma_j T}
\leq
\sum_{j \in [n]}
\gamma_{j}\delta_{j}(0)^2 e^{-4\gamma_j T}
=
\hat L(\delta(T))
= \eps\\
\implies
T &\geq
\frac{1}{4\gamma_k}\log(\frac{|S|\gamma_{j^*}\delta_{j^*}(0)^2}{\eps})
\label{eqn:T}
\end{align}

Decompose the train loss as:
$$
\hat L(\delta) =
\hat L_{\bar S}(\delta)
+
\hat L_S(\delta)
$$

Where 
$\hat L_S(\delta) := \sum_{i \in S} \gamma_i \delta_i^2$
and 
$\hat L_{\bar S}(\delta) := \sum_{i \in \bar S} \gamma_i \delta_i^2$.
Now, we can bound $\hat L_{\bar S}(\delta(T))$ as:

\begin{align*}
\hat L_{\bar S}(\delta(T))
=
\sum_{i \in \bar S} \gamma_i \delta_i(T)^2
&=
\sum_{i \in \bar S} \gamma_i \delta_i(0)^2 e^{-4 \gamma_i T}\\
&\leq 
\sum_{i \in \bar S} \gamma_i \delta_i(0)^2
(\frac{\eps}{|S|\gamma_{j^*}\delta_{j^*}(0)^2})^{\gamma_i/\gamma_k}
\tag{by Equation~\ref{eqn:T}}\\
&\leq 
\sum_{i \in \bar S} \gamma_i \delta_i(0)^2
(\frac{\eps}{|S|\gamma_{j^*}\delta_{j^*}(0)^2})^{1+p}\\
&\leq \alpha \eps
\end{align*}
Where the last inequality is due to Condition 2 of the Lemma.
Now, since
$\hat L(\delta) =
\hat L_{\bar S}(\delta)
+
\hat L_S(\delta)$,
and $\hat L(\delta(T)) = \eps$,
and 
$\hat L_{\bar S}(\delta(T)) \leq \alpha \eps$,
we must have
$$\hat L_S(\delta(T)) \geq (1-\alpha)\eps$$
This lower-bounds the population loss as desired.
\begin{align*}
L_\cD(\delta(T))
&\geq
\sum_{i \in S}\lambda_i \delta_i(T)^2
\geq
\min_{j \in S}(\frac{\lambda_j}{\gamma_j})
\sum_{i \in S}\gamma_i \delta_i(T)^2
=
\min_{j \in S}(\frac{\lambda_j}{\gamma_j})
\hat L_S(\delta(T))
\geq
\min_{j \in S}(\frac{\lambda_j}{\gamma_j})
(1-\alpha)\eps
\end{align*}

{\bf Annealed Gradient Descent.}

For annealed gradient descent, the proof idea is:
set the learning-rate such that in the gradient descent stage,
the optimizer oscillates on the first (highest-curvature) coordinate $\delta_1$,
and optimizes until the remaining coordinates are $\approx 0$.
Then, in the gradient flow stage, it optimizes on the first coordinate
until hitting the target train loss of $\eps$.
Thus, at completion most of the train loss will be due to contributions from
the ``large'' first coordinate.

Specifically, the gradient descent dynamics with stepsize $\eta$ is:
$$
\delta_i(t+1) = (1-2\eta \gamma_i)\delta_i(t)
$$
for discrete $t \in \N$.
We set $\eta := 1/\gamma_1$, so the first coordinate simply oscillates:
$\delta_1(t+1) = -\delta_1(t)$.
This is the case for the top eigenspace, i.e. all coordinates $i$ where
$\gamma_i = \gamma_1$.
The remaining coordinates decay exponentially.
That is, let $Q := \{i : \gamma_i \neq \gamma_1\}$
be the remaining coordinates (corresponding to smaller eigenspaces).
Then we have:
\begin{equation}
\label{eqn:exp-decay}
\forall i \in Q: ~\delta_i(t)^2 \leq c^{-t} \delta_i(0)^2
\end{equation}

for constant $c = \max_{j: \gamma_j \neq \gamma_1} |1-2\eta \gamma_j|^2 < 1$.

Recall, the population and empirical losses are:
\begin{align*}
    L_\cD(\delta) = \sum_i \lambda_i \delta_i^2 \quad ; \quad
    \hat L(\delta) = \sum_i \gamma_i \delta_i^2
\end{align*}
By Equation~\ref{eqn:exp-decay}, after $K$ steps of gradient descent, 
the contribution to the population
loss from the coordinates $i \in Q$ is small:
$$\sum_{i \in Q} \lambda_i \delta_i(K)^2 \leq \exp(-\Omega_{c}(K))$$

By Condition 1 of the Lemma, the train loss is still not below $\eps$
after the gradient descent stage, since the first coordinate 
did not optimize: $\delta_1(0) = \pm \delta_1(0)$.
We now run gradient flow, stopping at time $T \in \R$ when the train loss is $\eps$.
At this point, we have
$$\hat{L}(\delta(T)) \leq \eps
\implies \sum_{i \not\in Q} \gamma_i \delta_i(T)^2 \leq \eps$$
And thus, 
\begin{align*}
L_\cD(\delta(T))
&=
\sum_{i \not\in Q} \lambda_i \delta_i(T)^2 
+
\sum_{i \in Q} \lambda_i \delta_i(T)^2 \\
&\leq
(\max_{k \not\in Q}\frac{\lambda_k}{\gamma_k})
\sum_{i \not\in Q} \gamma_i \delta_i(T)^2 
+
\exp(-\Omega(K))\\
&\leq
(\max_{k \not\in Q}\frac{\lambda_k}{\gamma_k})
\eps
+
\exp(-\Omega(K))
\end{align*}
as desired.
\end{proof}

As a corollary of Lemma~\ref{thm:main}, we recover the 2-dimensional
of Claim~\ref{corr}.

\subsection{Proof of Claim~\ref{corr}}

\begin{proof}[Proof sketch of Claim~\ref{corr}]
Consider the distribution $\cD$ defined as:
Sample $x \in \{\vec{e_1}, \vec{e_2}\}$ uniformly at random, and let
$y = \langle \bs,  x \rangle$
for ground-truth $\bs = (100 \alpha^{-1/2}, 100 \alpha^{-1/2})$.
The population covariance is simply
$$\Sigma_X = 
\begin{bmatrix}
1/2 & 0\\
0 & 1/2
\end{bmatrix}
$$

With probability $3/4$, two of the three samples will have the same value of $x$.
In this case,
the sample covariance will be equal (up to reordering of coordinates) to
$$\hat{\Sigma}_X = 
\begin{bmatrix}
2/3 & 0\\
0 & 1/3
\end{bmatrix}
$$
This satisfies the conditions of Lemma~\ref{thm:main},
taking the ``small'' set of eigenvalue indices to be $S = \{2\}$.
Thus, the conclusion follows by Lemma~\ref{thm:main}.

With probability $1/4$, all the samples will all be identical,
and in particular share the same value of $x$.
Here, the optimization trajectory
will be 1-dimensional,
and it is easy to see that gradient flow and annealed gradient descent will reach identical minima.
\end{proof}

\end{document}